\documentclass[english,smallcondensed]{svjour3}
\usepackage[T1]{fontenc}
\usepackage[latin9]{inputenc}
\usepackage{geometry}
\usepackage{mathtools}
\usepackage{algorithm2e}
\usepackage{amsmath}
\usepackage{amssymb}
\usepackage{stackrel}
\usepackage[all]{xy}

\makeatletter
\numberwithin{figure}{section}
\numberwithin{equation}{section}

\makeatother

\usepackage{babel}

\begin{document}
\title{Scaling the weight parameters in Markov logic networks and relational
logistic regression models}
\author{Felix Q. Weitk\"amper\\
ORCID 0000-0002-3895-8279}
\institute{
Institut f\"ur Informatik\\
Ludwig-Maximilians-Universit\"at M\"unchen\\
Bundesrepublik Deutschland\\
felix.weitkaemper@lmu.de\\}
\maketitle

\begin{abstract}
Extrapolation with domain size has received plenty of attention recently, 
both in its own right and as part of the broader issue of scaling inference and learning to large domains.
We consider Markov logic networks and relational logistic regression
as two fundamental representation formalisms in statistical relational
artificial intelligence that use weighted formulas in their specification.
However, Markov logic networks are based on undirected graphs, while
relational logistic regression is based on directed acyclic graphs.
We show that when scaling the weight parameters with the domain size,
the asymptotic behaviour of a relational logistic regression model
is transparently controlled by the parameters, and we supply an algorithm
to compute asymptotic probabilities. We show using two examples
that this is not true for Markov logic networks. We also discuss using
several examples, mainly from the literature, how the application
context can help the user to decide when such scaling is appropriate
and when using the raw unscaled parameters might be preferable. 
\end{abstract}

\subsubsection*{Keywords }

Markov logic networks, Relational logistic regression, Scaling by
domain size, Bayesian networks

\newpage

\section{Introduction}

In the last 20 years, Statistical Relational Artificial Intelligence
(StarAI) has developed into a promising approach for combining the
reasoning skills of classic symbolic AI with the adaptivity of modern
statistical AI. It is not immediately clear, however, how StarAI behaves
when transitioning between domains of different sizes. This is a particularly
pertinent issue for StarAI since the modular design of statistical
relational formalisms, that are presented as a general template alongside
a given concrete domain, is one of their main attractions. Furthermore,
scalability and the prohibitive costs of learning and inference on
large domains have proven a key barrier to the widespread deployment
in applications, and therefore a transfer of learning results from
a smaller set of training data to a larger set of test data is especially
attractive. 

The topic of extrapolating across domain sizes has therefore received
some attention from the literature, and general patterns of behaviour
have become clear. On the one hand, Jaeger and Schulte (2018, 2020)
have provided very limiting necessary conditions under which domain
size does not affect inference in different StarAI approaches. On
the other hand, Poole et al. (2014) have characterised the extrapolation
behaviour of both Markov Logic Networks (MLNs) and Relational Logistic
Regression models (RLRs) on a small class of formulas on which the inferences
turn out to be asymptotically independent of the learned or supplied
parameters. This characterisation was extended and partly corrected
by Mittal et al. (2019) using considerable analytic and numerical
effort. They also present a proposal to mitigate the domain-size dependence
in MLNs by scaling the weights associated with formulas according to
the size of the domain, calling the resulting formalism Domain-size
Aware Markov Logic Networks (DA-MLNs). With similar computational effort,
they prove that asymptotic probabilities in DA-MLNs are dependent on
the supplied parameters for some example cases. However, a general
and systematic investigation of this dependence is still lacking. 

\subsection{Aims of the Paper}

This paper has three main objectives: First, we introduce a representation
of the probability of an atom in a grounding of an MLN as the integral
of a function $\mathrm{sigmoid}(\delta_{R(\vec{x})}^{T,n}(\mathbb{\mathfrak{X}}))$
directly derived from the weight parameters of the model. We then
use this tool to evaluate the asymptotic probabilities (with increasing
domain size) of two examples, and we see that neither MLNs nor the
scaled DA-MLNs prove adequate to ensure an asymptotic behaviour that
actually depends on the parameters of the model: $R(x)$ in the (DA-)MLN
given by $P\Rightarrow R(x):w$, and $Q(x)$ in the (DA-)MLN given
by $P\wedge Q(x)\wedge R(x,y)$. 

Second, we adapt domain-size dependent scaling of weights to a directed
analogue of MLNs, the RLRs introduced by Kazemi et al. (2014a, 2014b),
to obtain a formalism that we call Domain-size Aware Relational
Logistic Regression (DA-RLR) in analogy to DA-MLNs. We show that
the weight-independent asymptotic behaviour that was exemplified above
for MLNs and DA-MLNs does not occur in DA-RLRs; in fact, we provide an
algorithm to determine the asymptotic probabilities. To the best of
our knowledge, this is the first StarAI formalism for which such a
complete characterisation of the asymptotic probabilities from the
supplied parameters is known. 

Finally, the natural interpretation of scaled weight parameters
as weighting proportions, rather than raw numbers, of influencing
factors, allows us to investigate whether possible scenarios of changing
domain sizes are more adequately covered by a scaled or an unscaled
model. We will discuss with reference to examples of both real and
toy models from the existing literature how the context of a use case
can help a user to make that decision, and we will see that a particularly
important case of changing domain size, that of a model trained on
a random sample of a full data set, can be ideally represented by
a scaled model.

\subsection{\label{sec:LitReview}Prior Research on Extrapolation across Domain
Sizes}

The issues around changing domain sizes were noticed very early on
in the development of Statistical Relational AI. Already in (1998),
Jaeger discussed the existence of asymptotic probabilities in relational
Bayesian networks (another directed networks approach), making a connection
with finite model theory and infinite models of probabilistic theories.
This is indeed a very interesting connection, and it would be an exciting
direction for further work to investigate the implications the current
work on asymptotic probabilities has for infinite models (See Section
\ref{sec:FutWork}). 

However, he did not characterise the probabilities that relational
Bayesian networks would converge to, nor the conditions under which
those probabilities depend on the weights. This was first explicitly
isolated as a problem by Jain et al. (2010), who introduced \emph{Adaptive
Markov Logic Networks} (AMLNs) as a proposed solution. This was the
first suggestion to vary the weights given to the formulas of an MLN
as the domain size increases. While our approach can be seen as a
special case of a putative RLR-translation of AMLNs, the focus of this
work differs from that of Jain et al. (2010) in at least two important ways.
Firstly, we isolate one particular scaling function and investigate
its technical properties and its asymptotic behaviour in depth. Secondly,
while Jain et al. advocate learning the function from data in different
domains, we suggest considering the choice of scaling a semantic problem
that has as much do with how the data was obtained or is interpreted
as with the data itself. 

A detailed analysis of the behaviour of the uncorrected models was
undertaken by Poole et al. (2014) with reference to MLNs and the new
framework of RLRs introduced by Kasemi et al. (2014a,b). They undertook
a study of the asymptotic behaviour of MLNs and proved 0-1 laws for
a certain class of MLNs (corrected by Mittal et al. (2019)). On the
other hand, Jaeger and Schulte (2018) showed that a narrow class of
MLNs (corresponding to the $\sigma$-determinate MLNs that define infinite
models in the work of Singla and Domingos (2007)) are well-behaved and indeed
invariant under an increase in domain-size. 

Most recently, Mittal et al. (2019) refined the classification of
Poole et al. (2014) and suggested a putative solution to the problem
of probabilities changing with domain size. In their model of DA-MLNs,
which will be formally introduced in Section \ref{subsec:DA-MLN},
weights are changing according to a formula that relies explicitly
on the number of connections a formula could possibly induce. They
showed that for certain combinations of MLNs and queries, the asymptotic
probabilities depend in a non-trivial way on the weights if they are
considered as a DA-MLNs, but not, if they are considered as a straightforward
MLNs. We will see in Section \ref{sec:tools} in some examples that
this does not generalise across DA-MLNs and queries, and will suggest
a reason why this is unlikely to be solved by simply changing the
way the weights are recalibrated depending on the induced connections. 

Instead, we believe that a transition towards directed models will
overcome what we consider the main technical weakness of the DA-MLNs,
the aggregation function from the numbers of connections of different
literals in the same formula. In a directed model, we can focus our
attention on scaling with regard to the relation at the child node,
rather than having to aggregate different connections. We show that
when applied to the RLR approach, the scaling of the weighting with
the domain size has both a very natural interpretation and guarantees
weight-dependent behaviour as domain size increases.

\subsection{Extrapolation and Scalability}\label{subsec:Extrapolation}

Understanding the asymptotic behaviour of a statistical relational formalism can be a major contribution to scalability.
Despite the recent advance in lifted inference approaches, which exploit symmetries induced by the statistical relational formalism, the inherent complexity of the task provides a natural limit to the scalability of inference -- See the work of Jaeger (2000, 2015), Jaeger and Van den Broeck (2012) and Cozman and Maua (2016) for an overview of the results around this topic.

Asymptotic investigations provide for a different approach: Rather than providing a new method for exact inference, or for approximate inference that is guaranteed to be a good estimate on any domain, an asymptotic representation allows us to compute asymptotic probabilities, which are then good estimations on large domains, and whose error limits to 0 as domain sizes increases.
This enables a form of approximate inference in time constant with domain size; since the approximation is only valid asymptotically with increasing domain size, this does not contradict the known results regarding the complexity of approximate inference cited above.

DA-MLNs were originally introduced not for approximate inference, however, but for the parameter learning task.
Learning parameters involves many inference steps, exarcebating the scalability issues described above. Therefore, learning parameters from randomly sampled subsets of an intended training set has the potential to improve parameter learning times substantially.
If the asymptotic probabilities are convergent, then one would assume that sampling from sufficiently large subsets would give a good estimate of the optimal parameters, by the following procedure:

Sample substructures of domain size $m<n$, where $m$ is larger than the highest arity ocurring among the relation symbols of the DA-MLN and the arity of the queries we are typically interested in.
Find the parameters $\theta$ of $G_{\theta}'$ that maximise the sum of the log-likelihoods of the samples of size $m$. Now consider $G_{\theta}$.
By the asymptotic convergence of probabilities, if $n$ is sufficiently large,
these parameters maximise the likelihood of obtaining sampled substructures of size $m$ using $G_{\theta}$, including realisations of the typical queries.
This procedure is related to that of  Kuzelka et al. (2018), but while they consider samples of size $m$ as training data, they still learn with respect to a fixed sample size $n$.

However, estimating the parameters in this way only works if the asymptotic probabilities are truly dependent on the parameters. It is well-known from Poole et al. (2014) and Mittal et al. (2019) that for ordinary MLNs and RLRs, this is not generally the case. We see below that there are even DA-MLNs where the query probabilities do not vary asymptotically with parameters.

In our contribution, we provide an algorithm to compute asymptotic probabilities with respect to DA-RLRs and show that the asymptotic limit indeed depends on the parameters associated with the DA-RLRs.
We thereby contribute to applications to both parameter learning and inference on DA-RLRss.
Furthermore, via the transparent rescaling relationship of DA-RLRs to RLRs, they can be exploited for learning and inference in RLRs themselves.  

\section{Markov Logic Networks and Relational Logistic Regression}

In this section we will lay out our terminology and briefly present
the syntax and semantics of Markov Logic Networks and Relational Logistic
Regression. Both of these approaches combine statistical with logical
information, and both use weights to achieve this. However, MLNs are
based on \emph{undirected} graphical models (Markov Networks) while
RLRs are based on \emph{directed} graphical models (Bayesian Networks).
An insightful discussion on this distinction and its importance in
Statistical Relational AI can be found in the textbook of De Raedt et al. (2016). 

\subsection{Terminology}

All the frameworks that we consider in this paper use the terminology
of relational first-order logic in both their syntax and the definition
of their semantics. While we repeat it here in brief, it can be found
in most standard logic textbooks. 

A \emph{(relational) signature }$\mathcal{L}$ is given by a set of\emph{ relation
symbols} $R$, each of a given natural number arity, and a set of \emph{constants} $a$. A relation symbol of arity 0 is known as a \emph{proposition}.
We always assume that our signature contains the propositions $\top$ and $\bot$.
In a \emph{multi-sorted }signature, each constant is annotated with a \emph{sort label}, and each relation symbol is annotated
with a list of  sort labels of length equal to its arity. The \emph{language}
corresponding to a given signature additionally has infinitely many
variables for each sort. An \emph{$\mathcal{L}$-term} is either a constant or a variable. An \emph{$\mathcal{L}$-atom }is a relation symbol together with a
list of terms (whose sorts correspond to the list of sorts with the
relation symbol). An \emph{$\mathcal{L}$-literal} is then given by either an atom
or a negated atom, where negation is indicated by $\neg$. A \emph{(quantifier-free)
$\mathcal{L}$-formula} $\varphi$ is defined recursively from literals using the binary connectives
$\wedge$, $\vee$ and $\rightarrow$. 
An \emph{$\mathcal{L}$-structure} $\mathfrak{X}$  is given by a multi-sorted \emph{domain} $D$ (a
set of individual elements for each of the sorts of the signature)
and an interpretation of each of its symbols, that is, an element of the correct sort for every constant and a set of lists of
elements of the correct sorts for each relation symbol.
In all these terms, the signature will be omitted where it is clear from context.

If $\mathcal{L}\subseteq\mathcal{L}'$
are two signatures, then we can consider any $\mathcal{L}'$-structure
as an $\mathcal{L}$-structure simply by omitting the interpretations
of the symbols not in $\mathcal{L}$. In this situation, the $\mathcal{L}$-structure
is called the \emph{reduct} and the $\mathcal{L}'$-structure is called
the \emph{extension}. 

A formula is \emph{grounded }by substituting elements of the appropriate
domains for its variables, and it is \emph{ground
}if it does not (any longer) contain variables. Therefore, any choice
of elements from the domains matching the sorts of the variables in
a formula is a \emph{possible grounding }of that formula. 

A ground atom \emph{holds} or \emph{is true }if the substituted list
of elements lies in the interpretation
of the relation symbol. $\top$ holds in every interpretation and $\bot$ holds in no interpretation. 
 Whether a formula \emph{holds }is then determined
by giving $\neg$, $\wedge$, $\vee$ and $\rightarrow$ their usual
meanings as 'not', 'and', 'or' and ``if ... then ...'' respectively. 

We will adopt the notation $\vec{x}$ for a tuple $(x_{1},x_{2},\ldots)$
whose length we do not wish to specify, and we may write $\vec{x}\in D$
when all entries in $\vec{x}$ are in $D$. We will further use the
notation $\left|\varphi(\vec{x})\right|_{\mathfrak{X}}$for the number
of all true possible groundings of $\varphi(\vec{x})$ in $\mathfrak{X}$
(the subscript will be omitted where the structure is clear from the
context). 

Finally, let us set some notation for directed acyclic graphs that
we will use later on: A \emph{cycle }is a path from a node to itself
in the directed graph, a \emph{loop }is a path from a node to itself
in the underlying undirected\emph{ }graph. A \emph{directed acyclic
graph }(DAG) $G$ is a directed graph without cycles, a \emph{polytree}
is a directed acyclic graph without loops. Nodes without parents are
called \emph{roots}, nodes without children are called \emph{leaves}.
The length of the longest path from a root node to a given node is
the latter node's \emph{index.}

\subsection{Markov Logic Networks}

As MLNs have been extensively discussed in the literature, we will
use this subsection mainly to set some notation. A more complete discussion
can be found in the original paper by Richardson and Domingos (2006)
or the textbook of De Raedt et al. (2016).

Therefore we will also restrict the definition of MLNs to a setting
large enough to accommodate all the examples in this work and in the
papers mentioned in Section \ref{sec:LitReview}
\begin{definition}
\label{def:MLNSyn}Let $\mathcal{L}$ be a (potentially multi-sorted)
relational signature. A \emph{Markov Logic Network T }over $\mathcal{L}$
is given by a collection of pairs $(\varphi_{i},w_{i})$ (called \emph{weighted
formulas}) where $\varphi$ is a quantifier-free $\mathcal{L}$-formula
and $w\in\mathbb{R}$. We call $w$ the \emph{weight} of $\varphi$
in $T$. 
\end{definition}
\begin{example}
As a running example for the next subsections, we will consider a
signature with two unary relation symbols $Q$ and $R$. Then one
could build an MLN consisting of just one formula, $R(x)\wedge Q(y):w$.
If the signature is single-sorted, one should distinguish it from
the MLN $\{R(x)\wedge Q(x):w\}$, where the variables are the same.
An example domain for the single-sorted signature would be the
three-element domain $\{1,2,3\}$,  which gives rise to
9 possible groundings of the formula  $R(x)\wedge Q(y):w$,
one for each choice of $x$ and $y$. 

\end{example}
Markov Logic Networks have been introduced by Richardson and Domingos
(2006) and have since then been highly influential in the field of
Statistical Relational AI. Their semantics is based on undirected
networks, which means that any literal in a formula can influence
any other in a dynamic way. We can obtain such a semantics for an
MLN $T$ by choosing a (finite) domain for each sort of $\mathcal{L}$.
Given such a choice, we will define the semantics as a probability
distribution over all $\mathcal{L}$-structures on the chosen domains
as follows:
\begin{definition}
\label{def:MLNSem}Given a choice $D$ of domains for the sorts of
$\mathcal{L}$, an MLN $T$ over $\mathcal{L}$ defines a probability
distribution on the possible $\mathcal{L}$-structures on the chosen
domains as follows: let $\mathfrak{X}$ be an $\mathcal{L}$-structure
on the given domains. Then 
\[
\mathcal{P}_{T,D}(\mathfrak{X})=\frac{1}{Z}\exp(\underset{i}{\sum}w_{i}n_{i}(\mathfrak{X}))
\]
where $i$ varies over all the weighted formulas in $T$, $n_{i}(\mathfrak{X})$
is the number of true groundings of $\varphi_{i}$ in $\mathfrak{X}$,
$w_{i}$ is the weight of $\varphi_{i}$ and $Z$ is a normalisation
constant to ensure that all probabilities sum to 1. 

As the probabilities only depend on the sizes of the domains, we can
also write $\mathcal{P}_{T,n}$ for domains of size $n$ when the
signature is single-sorted. 

We refer to $\underset{i}{\sum}w_{i}n_{i}(\mathfrak{X})$ as the \emph{weight
of $\mathfrak{X}$} and write it as $w_{T}(\mathfrak{X})$. 
\end{definition}
\begin{example}
In the MLN $R(x)\wedge Q(x):w$, the probability of any possible structure
$\mathfrak{X}$ with domain $D$ is proportional to $\exp\left(w\cdot n(\mathfrak{X})\right)$,
where $n(\mathfrak{X})$ is the number $\left|R(x)\wedge Q(x)\right|$
of elements $a$ of $D$ for which $R(a)$ and $Q(a)$ hold in the
interpretation from $\mathfrak{X}$. In the MLN $w:R(x)\wedge Q(y)$,
however, this probability is proportional to $\exp\left(w\cdot n'(\mathfrak{X})\right)$,
where $n'(\mathfrak{X})$ is the number of pairs $(a,b)$ from $D\times D$
for which $R(a)$ and $Q(b)$ hold in the interpretation from $\mathfrak{X}$.
In other words, $n'(\mathfrak{X})$ is the product $\left|R(x)\right|\cdot\left|Q(y)\right|$.
\end{example}

\subsection{Domain-size Aware Markov Logic Networks\label{subsec:DA-MLN}}

Mittal et al. (2019) have introduced weight scaling to MLNs in order
to compensate for the effects of variable domain sizes. They call
the resulting formalism \emph{Domain-size Aware Markov Logic Networks
(DA-MLNs)} and we will rehearse the main definitions from their paper
here.
\begin{definition}
A \emph{Domain-size Aware Markov Logic Network (DA-MLN) }is given
by the same syntax as a regular MLN (see Definition \ref{def:MLNSyn}).
\end{definition}
In order to adapt the semantics to changing domain size, Mittal et
al. (2019) use the concept of a \emph{connection vector}.
\begin{definition}
Let $\varphi$ be a formula. Let $\Psi$ be the set of literals of
$\varphi$ and for every $\psi\in\Psi$, let $V_{\psi}$ be the set
of free variables in $\varphi$ not occurring in $\psi$. For every
variable $x$, let $D_{x}$ signify its domain. Then the \emph{connection
vector} of $\varphi$ is the set $\{\underset{x\in V_{\psi}}{\prod}\left|D_{x}\right||\psi\in\Psi\}$. 
\end{definition}
\begin{example}
The connection vector of the formula $R(x)\wedge Q(x)$ is given by
the set $\{1,1\}=\{1\}$, since the same variables occur in both literals
and therefore both products in the definition are empty. The connection
vector of the formula $R(x)\wedge Q(y)$ is given by the set $\{\left|D_{y}\right|,\left|D_{x}\right|\}$,
since $V_{R(x)}=y$ and $V_{Q(y)}=x$. If the signature is single-sorted,
$\left|D_{y}\right|=\left|D_{x}\right|$ and the connection vector simplifies to $\{\left|D_{x}\right|\}$. 
\end{example}
The connection vector records how many tuples each literal could possibly
connect to; in the formula $P(x)\wedge Q(x,y)\wedge R(z)$, for instance,
the literal $P(x)$ could connect to $|D_{y}|{\cdot}|D_{z}|$ many tuples,
the literal $Q(x,y)$ could connect to $|D_{z}|$ many tuples and
the literal $R(z)$ could connect to $|D_{x}|{\cdot}|D_{z}|$ many tuples.
We will see several examples later where even in a single-sorted structure
the connection vector can contain different elements. 

The problem now is to aggregate this information into a single scaling
factor. Mittal et al. (2019) use the maximum of the entries of the
connection vector, but suggest investigating other options as well
(see Subsection \ref{subsec:AggrFunc} below). 
\begin{definition}
Given a choice of domains for the sorts of $\mathcal{L}$, a DA-MLN
$T$ over $\mathcal{L}$ defines a probability distribution on the
possible $\mathcal{L}$-structures on the chosen domains as follows:
let $\mathfrak{X}$ be an $\mathcal{L}$-structure on the given domains.
Then 
\[
\mathcal{P}_{T,D}(\mathfrak{X})=\frac{1}{Z}\exp(\underset{i}{\sum}\frac{w_{i}}{C_{i,D}}n_{i}(\mathfrak{X}))
\]
where $n_{i}(\mathfrak{X})$ is the number of true groundings of $\varphi_{i}$
in $\mathfrak{X}$, $w_{i}$ is the weight of $\varphi_{i}$, $C_{i,D}$
is the maximum of the entries of the connection vector of $\varphi_{i}$,
and $Z$ is a normalisation constant to ensure that all probabilities
sum to 1.
If the connection vector is empty, then $C_{i,D}$ is set to be 1. 
\end{definition}
\begin{example}
For $R(x)\wedge Q(x)$, $C_{D}=1$, while for $R(x)\wedge Q(y)$,
$C_{D}=\max\left(\left|D_{y}\right|,\left|D_{x}\right|\right)$ . 
\end{example}
Mittal et al. (2019) give several examples of how moving from ordinary
to DA-MLNs changes the asymptotic behaviour; we will see further examples
in Section \ref{sec:tools} below.

\subsection{Relational Logistic Regression}

Relational Logistic Regression differs from MLNs in that it is based
on directed rather than undirected models. Thus rather than allowing
arbitrary weighted formulas, one first specifies a \emph{Relational
Belief Network (RBN). }A detailed exposition of RLRs is given by Kazemi
et al. (2014b), and we will restrict ourselves here to the case of
a relational signature without constants. 
\begin{definition}
\label{def:RLRSyntax}A Relational Belief Network (RBN) over a (relational)
signature (without constants) $\mathcal{L}$ is a directed acyclic
graph whose nodes are relation symbols from $\mathcal{L}$ and in
which every relation symbol appears exactly once.

A \emph{Relational Logistic Regression $T$} over a relational signature
$\mathcal{L}$ consists of an RBN $G$ whose nodes $R$ are labelled
with a pair $(\varphi,(\psi_{i},w_{i},V_{i})_{i})$, where $\varphi$
is an $\mathcal{L}$-atom starting with $R$ and $(\psi_{i},w_{i},V_i)_{i}$
is a set of triples such that $V_{i}$ is a finite set of variable
symbols, $\psi_{i}$ a quantifier-free formula and $w_{i}\in\mathbb{R}$.
Furthermore, the following are required to hold: 
\begin{enumerate}
\item None of the variable symbols appearing in $\varphi$ are in a $V_{i}$
\item Only relation symbols from parents of $R$ occur in $\psi_{i}$, and
every variable in $\psi_{i}$ appears either in $\varphi$ or is in
$V_{i}$. If $R$ is a root node, then there is only one possible
formula, with $\psi=\top$ and $V_{i}$ is set
to be $\emptyset$.
\end{enumerate}
Furthermore, for any formula $\psi(x)$, let the \emph{index of }$\psi$
with respect to a given or implied RBN be defined as the largest index
of any relation symbol ocurring in $\psi$.
\end{definition}
\begin{example}\label{ExampleRLR}
Consider the single-sorted signature $\{Q,R\}$ already used in the
examples of the last subsections. Then we can form the RBNs $Q\longrightarrow R$
and $R\longrightarrow Q$. Choosing the latter, we now have to choose
a weight $w_R$ for $R(x)$. This determines the probability of $R(a)$
to hold for any $a\in D$. At $Q$, take $\varphi\coloneqq Q(x)$.
Then for a fixed weight $w_Q$ we can distinguish an RLR with $\psi\coloneqq R(x)$ (and $V_{i}\coloneqq\emptyset$)
and an RLR with $\psi\coloneqq R(y)$ (and $V_{i}\coloneqq\{y\}$).
\end{example}
While the semantics of the MLNs in Definition \ref{def:MLNSem} was
given as a single probability distribution, we will define the semantics
of the RLR by recursion over the underlying RBN. More precisely, the
definition will proceed by recursion over the index of the relation
symbols as nodes of the RBN. To facilitate writing, we will introduce
some notation before defining the semantics of an RLR: 
\begin{definition}
\label{def:RLRSemantics}The $\mathrm{sigmoid}$ function is defined
by $\mathrm{sigmoid}(k)\coloneqq\frac{\exp(k)}{\exp(k)+1}$. We will
furthermore use the expression $1_{\psi}$ for the function that takes
a structure as input and returns $1$ whenever $\psi$ holds in that structure and $0$
otherwise. Finally, we will use the convention of denoting with $\psi(\vec{a}/\vec{x})$
the grounding of $\psi$ where $\vec{a}$ is substiuted for the variables
$\vec{x}$. 

Let $\mathcal{L}_{n}$ be defined as the subsignature of $\mathcal{L}$
consisting of all those relations that label a node of index smaller
than $n$. Then we will define a probability distribution on the possible
$\mathcal{L}_{n}$-structures on a given domain $D$ for $\mathcal{L}$
by induction over $n$ and as the product of the probabilities for
every grounding of the atom:

$n=0$: The probability of any given grounding of the atom in $\mathcal{L}_{0}$
at node $O$ is given by $\mathrm{sigmoid}(w)$, the probability of
its negation thus by $1-\mathrm{sigmoid}(w)$. 

Assume now that a probability distribution on $\mathcal{L}_{n}$-structures
has been defined. We will now extend it to $\mathcal{L}_{n+1}$-structures
as follows: The probability of an $\mathcal{L}_{n+1}$-structure is
given by the probability of its $\mathcal{L}_{n}$-reduct multiplied
with the conditional probability of the groundings of the atoms in
$\mathcal{L}_{n+1}\backslash\mathcal{L}_{n}$. These are given by
\[
\mathcal{P}(Q(\vec{x}))=\mathrm{sigmoid}(\underset{i}{\sum}w_{i}\underset{\vec{a}\in D}{\sum}1_{\psi_{i}(\vec{a}/V_{i}))})
\]
Note that the right-hand side of this equation only depends on the
$\mathcal{L}_{n}$-reduct because of the conditions on the $\psi_{i}$
from the definition of an RLR above. 
\end{definition}
\begin{example}
Consider the two RLRs $R\rightarrow Q$ introduced in the example
above. Then the probability of $R(a)$ is $\mathrm{sigmoid}(w)$ in
each case. Where $Q(x)$ is annotated with $R(x):w'$, the probability
of $Q(a)$ for any given element $a$ is $\mathrm{sigmoid}(w'\cdot1_{R(a)})$,
which is $\mathrm{sigmoid}(w')$ if $R(a)$ is true and $\frac{1}{2}$
otherwise. Where $Q(x)$ is annotated with $R(y):w'$, the probability
of $Q(a)$ for any given element $a$ is given by $\mathrm{sigmoid}(w'{\sum}_{b\in D}1_{R(b)})=\mathrm{sigmoid}(w'\cdot m)$,
where $m$ is the number of domain elements $b$ such that $R(b)$
is true. In particular, it does not depend on whether $R(a)$ itself
is true or not.
\end{example}

\section{\label{sec:tools}Asymptotic Probabilities in MLNs and DA-MLNs}

In this section we will build on the work of Mittal et al. (2019)
and give two examples of dependencies for which an asymptotic behaviour
that depends on the weight is achieved neither in MLNs nor in DA-MLNs.
In order to give a clear and rigorous structure to our derivations,
we will first introduce probability kernels for Markov logic and prove
some basic facts about their behaviour under limits. Then we continue
to give a characterisation of asymptotic probabilities in some MLNs
and DA-MLNs. To formalise this discussion and make it more amenable
to calculation, we will use measure and integral notation for the
probability distributions arising from MLNs. We will use $\mu_{T,D}$
to refer to the probability measure induced by an MLN $T$ on a choice
of domains $D$, which can again be replaced by $\mu_{T,\vec{n}}$ where
the domains of the sorts $s$  are of cardinality $n_s$.

\subsection{\label{subsec:Probability-kernels}Characterising Probabilities in
MLN and DA-MLN}

In order to determine the probability of a certain ground atomic formula
being true, we compute the probability measure of the set of all worlds
in which this formula holds. In many cases, e. g.\ if the signature
itself does not contain any constants, this probability is independent
of the domain elements used in grounding the formula. In that case,
we will write $\mathcal{P}(R(\vec{x}))$ for a tuple of variables
to indicate that the choice of grounding is immaterial.
This can also be conveniently written as an integral 
\[
\mathcal{P}_{T,D}(R(\vec{a}))={\int}1_{R(\vec{a})} d{\mu_{T,D}}
\]
 We would like to replace the indicator function in the integral with
a function that we can express analytically in terms of the weights
of the MLN model. We can achieve that by considering the weighted
mean of the indicator functions between two models that only differ
in the value of $R(\vec{a})$.

Let $\mathfrak{X}$ be a structure and let $\vec{a}\in\mathfrak{X}$
be a tuple. Then let $\mathfrak{X}_{R(\vec{a})}$ be the structure
that potentially differs from $\mathfrak{X}$ only in that $R(\vec{a})$
holds in $\mathfrak{X}_{R(\vec{a})}$, regardless of whether it holds
in $\mathfrak{X}$. $\mathfrak{X}_{\neg R(\vec{a})}$ is defined analogously.
In this way the class of all structures of a given domain size divides
equally in structures of the form $\mathfrak{X}_{R(\vec{a})}$ and
structures of the form $\mathfrak{X}_{\neg R(\vec{a})}$, and there
is a natural one-to-one correspondence betwen structures of each type.
Therefore we can compute the integral above by instead considering
the weighted mean of the indicator function on both $\mathfrak{X}_{R(\vec{a})}$and
$\mathfrak{X}_{\neg R(\vec{a})}$. This mean can be described in terms
of an auxiliary function that reflects the dependence of the weights
on the validity of $R(\vec{x})$:
\begin{definition}
Let $\mathfrak{X}$ be a structure of domain size $n$, $\vec{a}\in T$
and let $T$ be a (DA-)MLN. Then
\[
\delta_{R(\vec{a})}^{T,D}(\mathbb{\mathfrak{X}})\coloneqq\underset{i}{\sum}w_{i}n_{i}(\mathfrak{X}_{R(\vec{a})})-\underset{i}{\sum}w_{i}n_{i}(\mathfrak{X}_{\neg R(\vec{a})}).
\]
\end{definition}
The weighted mean of $1_{R(\vec{a})}$ on $\mathfrak{X}_{R(\vec{a})}$
and $\mathfrak{X}_{\neg R(\vec{a})}$ can now be computed as follows:
\begin{proposition}
\label{prop:sigmoid}For any MLN $T$, any structure $\mathfrak{X}$
on a choice of domains $D$ and any ground atom $R(\vec{a})$, the
$\mu_{T,D}$-weighted mean of $1_{R(\vec{a})}$ on $\mathfrak{X}_{R(\vec{a})}$
and $\mathfrak{X}_{\neg R(\vec{a})}$ is given by $\mathrm{sigmoid}(\delta_{R(\vec{a})}^{T,D}(\mathbb{\mathfrak{X}}))$. 
\end{proposition}
\begin{proof}
\begin{align*}
\frac{1\cdot\mu(\mathfrak{X}_{R(\vec{a})})}{\mu(\mathfrak{X}_{R(\vec{a})})+\mu(\mathfrak{X}_{\neg R(\vec{a})})} & =\frac{1\cdot\left[\frac{\mu(\mathfrak{X}_{R(\vec{a})})}{\mu(\mathfrak{X}_{\neg R(\vec{a})})}\right]}{\left[\frac{\mu(\mathfrak{X}_{R(\vec{a})})}{\mu(\mathfrak{X}_{\neg R(\vec{a})})}\right]+1}=\frac{\left[\frac{\exp(\underset{i}{\sum}w_{i}n_{i}(\mathfrak{X}_{R(\vec{a})}))}{\exp(\underset{i}{\sum}w_{i}n_{i}(\mathfrak{X}_{\neg R(\vec{a})}))}\right]}{\left[\frac{\exp(\underset{i}{\sum}w_{i}n_{i}(\mathfrak{X}_{R(\vec{a})}))}{\exp(\underset{i}{\sum}w_{i}n_{i}(\mathfrak{X}_{\neg R(\vec{a})}))}\right]+1}=\\
=\frac{\exp(\delta_{R(\vec{a})}^{T,D}(\mathbb{\mathfrak{X}}))}{\exp(\delta_{R(\vec{a})}^{T,D}(\mathbb{\mathfrak{X}}))+1} & =\mathrm{sigmoid}(\delta_{R(\vec{a})}^{T,D}(\mathbb{\mathfrak{X}}))
\end{align*}
Using this characterisation, we can express the probability of $R(\vec{a})$
as an integral over $\delta$, which is directly connected to the
weights in the (DA-)MLN.
\end{proof}

\begin{corollary}\label{cor:ProbEqn}
For any MLN $T$, any choice of domains $D$ and any ground atom $R(\vec{a})$,
\[
\mathcal{P}_{T,D}(R(\vec{a})) = {\int}\mathrm{sigmoid}(\delta_{R(\vec{a})}^{T,D}) d{\mu_{T,D}}.
\]
\end{corollary}
\begin{remark}
Note that Corollary \ref{cor:ProbEqn} is an implicit characterisation or property of the
probabilities induced by the MLN, rather than an alternative definition.
In particular, the probabilities from the MLN $T$ occur on both sides
of the equality, once as the $\mathcal{P}_{T,D}(R(\vec{a}))$ and
once as the $\mu_{T,D}$.
\end{remark}
The key tool that makes Corollary \ref{cor:ProbEqn} usable in concrete calculations is the classical law
of large numbers from probability theory:
\begin{proposition}
The strong law of large numbers: Let $\left(X_{n}\right)_{n\in\mathbb{N}}$
be a sequence of real, integrable, identically distributed and pairwise
independent random variables. Then the sequence $\frac{1}{n}({\sum}_{i = 1 \dots n}X_{i})$
converges to the expected value of $X_{1}$ almost surely. This applies
in particular when $X_{i}$ are sequences of identically distributed
independent Bernoulli trials. 
\end{proposition}
We also derive a lemma which we will need at several points. 
Later we will use a much sharper version of the same idea in our exact
treatment of asymptotic probabilities in domain-size aware directed
models.
\begin{lemma}
\label{lem:Weak-law}Let $\varphi$ be a formula with at least one
free variable. If, for all $\vec{a}\in D$ for any $D$ of any size
and any structure $\mathfrak{X}$ on $D$, $\mathrm{sigmoid}(\delta_{R(\vec{a})}^{T,D})\geq k\in(0,1)$,
then for all $N\in\mathbb{N}$ and every $\varepsilon>0$, there is
an $n\in\mathbb{N}$ such that for any domain $D$ larger than $n$,
the probability that $\varphi$ has less than $N$ true groundings
in $D$ is less than $\varepsilon$.
\end{lemma}
\begin{proof}
In any sufficiently large domain, consider the events $Y_{i}\coloneqq\varphi(\vec{a}_{i})$
for $\vec{a}_{i}\in D$ as Boolean random variables. Then by $\mathrm{sigmoid}(\delta_{R(\vec{a})}^{T,D})\geq k$,
for each $i\in\mathbb{N}$, the conditional probability of success
of $Y_{i}$ given any sequence of outcomes of the $Y_{j}$ for $j<i$
is at least $k$. Consider the sequence of random variables $(X_{i})_{i\in\mathbb{N}}$
where $X_{i}$ is a Bernoulli trial with success probability $k$.
Since the conditional probability of success of $Y_{i}$ given any
sequence of outcomes of the $Y_{j}$ for $j<i$ is at least $k$,
we can consider $\left(Y_{i}\right)_{i\in\mathbb{N}}$ to be stochastically
dominated by $\left(X_{i}\right)_{i\in\mathbb{N}}.$ However, by the
law of large numbers, $\frac{1}{n}{\sum}_{i = 1 \dots n} X_{i})$
converges to $k$ almost surely and therefore, in particular the sequence
$X_{1},\ldots,X_{N}$ has more than $n$ successes with probability
above $1-\varepsilon$ when the $N$ is sufficiently large. A fortiori
this also holds for $\left(Y_{i}\right)_{i\in\mathbb{N}}$.
\end{proof}
We now turn to investigating the asymptotic behaviour of several MLNs and DA-MLNs. 

\subsection{The Formula $P\rightarrow R(x)$\label{subsec:Example1} }

The formula $P\rightarrow R(x)$ is the converse of $R(x)\rightarrow P$,
possibly the most studied formula in discussions on varying domain
sizes and asymptotic probability. Poole et al. (2014) give a detailed
analysis of the behaviour of $P$ in this configuration, but we will
focus here on the asymptotic behaviour of $R(x)$. 
\begin{proposition}\label{prop:MLN-Ex-1}
  Let $T_{w}$ be the MLN given by the formula
$P\rightarrow R(x):w$, $w>0$, $|D_{x}|=n$. Then the asymptotic
probability of $P$ is 0 and the asymptotic probability of $R(x)$
is $\frac{1}{2}$, independent of the value of $w$.
\end{proposition}
\begin{proof}
We will use our results from the preceding subsection. We observe that
\begin{align*}
\mathcal{P}_{T_{w},n}(R(x))={\int} \mathrm{sigmoid}(\delta_{R(x)}^{T_{w},n}) d{\mu_{T_{w},n}} & = {\int}\mathrm{sigmoid}(w\cdot1_{P}) d{\mu_{T_{w},n}} \leq {\int}\mathrm{sigmoid}(w) d{\mu_{T_{w},n}} =\mathrm{sigmoid}(w).
\end{align*}
Thus, Lemma \ref{lem:Weak-law} holds and for all $N\in\mathbb{N}$
there is an $n\in\mathbb{N}$ such that $\mathcal{P}(|\neg R(\mathfrak{X})|>N)\geq1-\varepsilon$.
Therefore we see that, for sufficiently large $n$,
\[
{\int}\mathrm{sigmoid}(-w\cdot|\neg R(x)|_{\mathfrak{X}}) d{\mu_{T_{w},n}} \leq (1-\varepsilon)\mathrm{sigmoid}(-w\cdot N)+\varepsilon,
\]
which converges to $0$ as $n$ grows to infinity and $\varepsilon$
approaches $0$. Thus,
\[
\underset{n\rightarrow\infty}{\lim}\mathcal{P}_{T_{w},n}(P)=\underset{n\rightarrow\infty}{\lim}{\int}\mathrm{sigmoid}(\delta_{P}^{T_{w},n}) d{\mu_{T_{w},n}} =\underset{n\rightarrow\infty}{\lim}{\int}\mathrm{sigmoid}(-w\cdot|\neg R(x)|_{\mathfrak{X}}) d{\mu_{T_{w},n}} = 0.
\]
This proves the first part of the proposition. We will now use this
to prove the second part. Since the asymptotic probability of $P$
is 0, there is for every $\varepsilon>0$ an $N\in\mathbb{N}$ such
that $\mathcal{P}_{T_{w},n}(P)<\varepsilon$. Thus,
\[
\underset{n\rightarrow\infty}{\lim}\mathcal{P}_{T_{w},n}(R(x)) = \underset{n\rightarrow\infty}{\lim}{\int}\mathrm{sigmoid}(\delta_{R(x)}^{T_{w},n}) d{\mu_{T_{w},n}} = \underset{n\rightarrow\infty}{\lim}{\int}\mathrm{sigmoid}(w\cdot1_{P})=\mathrm{sigmoid}(0) d{\mu_{T_{w},n}} = \frac{1}{2}.
\]
\end{proof}

As $P\rightarrow R(x)$ is logically equivalent to $\neg R(x)\rightarrow\neg P$,
we obtain analogous results for $R(x)\rightarrow P$:
\begin{corollary}
Let $T_{w}$ be the MLN given by the formula $R(x)\rightarrow P:w$,
$w>0$, $|D_{x}|=|D_{y}|=n$. Then the asymptotic probability of $R(x)$
is $\frac{1}{2}$ and the asymptotic probability of $P$ is 1, independent
of the value of $w$. 
\end{corollary}
Of course, it is well known that MLNs do not generally behave well
asymptotically. Usually, however, this phenomenon has been observed
in atoms that have an unbounded number of connections themselves.
Here, the issue stems from the fact that the atom with which it is
connected has an unbounded number of connections and therefore degenerates. 

Since moving to DA-MLNs will regulate the asymptotic behaviour of $P$,
one might expect that this will also make the probability of $P$
weight-dependent. However, unfortunately, scaling the weights with
the domain size will itself impact $R(x)$ in the same way:
\begin{proposition}
Let $T_{w}$ be the DA-MLN given by the formula $P\rightarrow R(x):w$,
$w>0$, $|D_{x}|=n$. Then the asymptotic probability of $R(x)$ is
$\frac{1}{2}$, independent of the value of $w$.
\end{proposition}
\begin{proof}
The connection vector for the formula $P\rightarrow R(x)$ is $(n,1)$,
and so weights will be scaled using the factor $n$.
\begin{align*}
\underset{n\rightarrow\infty}{\lim}\mathcal{P}_{T_{w},n}(R(x)) & =\underset{n\rightarrow\infty}{\lim}{\int}\mathrm{sigmoid}(\delta_{R(x)}^{T_{w},n}) d{\mu_{T_{w},n}} =\underset{n\rightarrow\infty}{\lim}{\int}\mathrm{sigmoid}(\frac{w}{n}\cdot1_{P}) d{\mu_{T_{w},n}}\\
 & \leq\underset{n\rightarrow\infty}{\lim}\mathrm{sigmoid}(\frac{w}{n})=\mathrm{sigmoid}(0)=\frac{1}{2}
\end{align*}
\end{proof}

Again, we obtain a corollary on $R(x)\rightarrow P:w$.
\begin{proposition}
Let $T_{w}$ be the DA-MLN given by the formula $R(x)\rightarrow P:w$,
$w>0$, $|D_{x}|=n$. Then the asymptotic probability of $R(x)$ is
$\frac{1}{2}$, independent of the value of $w$.
\end{proposition}

We see that in fact, neither semantics displays adequate, weight-dependent scaling behaviour for this weighted formula. 

\subsection{The Formula $P\wedge Q(x)\wedge R(x,y)$\label{subsec:Example2}}

In this subsection we will discuss the formula $P\wedge Q(x)\wedge R(x,y)$
as an example in which DA-MLNs overcompensate for the amount
of connections of one literal, namely $Q(x)$.

In ordinary MLNs, the asymptotic probability of $Q(x)$ is 1, regardless
of the weight. 
\begin{proposition}
Let $T_{w}$ be the MLN given by the formula $P\wedge Q(x)\wedge R(x,y):w$,
$w>0$, $|D_{x}|=|D_{y}|=n$. Then the asymptotic probability of $Q(x)$
is 1, independent of the value of $w$.
\end{proposition}
\begin{proof}
First observe that on all structures, $\delta_{P}^{T_{w},n}$ will
be at least $\frac{1}{2}$:
\[
\mathrm{sigmoid}(\delta_{P}^{T_{w},n})=\mathrm{sigmoid}(w\cdot|Q(x)\wedge R(x,y)|)\geq\mathrm{sigmoid}(0)=\frac{1}{2}
\]
We will now establish that, for any fixed $x$, the probability of
$P\wedge R(x,y)$ is always at least $\frac{1}{4}$. This is a consequence
of the Bayesian formula, which implies that $P\wedge R(x,y)=\mathcal{P}_{T_{w},n}(P){\cdot}\mathcal{P}_{T_{w},n}(R(x,y)|P)$.
Since the analysis of Subsection \ref{subsec:Probability-kernels}
is just as valid for conditional probabilities, we can derive
\begin{align*}
\mathcal{P}_{T_{w},n}(R(x,y)|P) & =\frac{{\int}\mathrm{sigmoid}(\delta_{R(x,y)}^{T_{w},n})\cdot1_{P} d{\mu_{T_{w},n}}}{\mathcal{P}_{T_{w},n}(P)}=\frac{{\int}\mathrm{sigmoid}(w\cdot1_{Q(x)})\cdot1_{P} d{\mu_{T_{w},n}}}{\mathcal{P}_{T_{w},n}(P)}\geq\\
\geq & \frac{{\int}\mathrm{sigmoid}(0)\cdot1_{P} d{\mu_{T_{w},n}}}{\mathcal{P}_{T_{w},n}(P)}=\frac{\frac{1}{2}\mathcal{P}_{T_{w},n}(P)}{\mathcal{P}_{T_{w},n}(P)}=\frac{1}{2}
\end{align*}
and therefore $\mathcal{P}_{T_{w},n}(P\wedge R(x,y))\geq\frac{1}{2}\cdot\frac{1}{2}=\frac{1}{4}$.
Thus we can once again apply Lemma \ref{lem:Weak-law} and conclude
that, asymptotically, there will be arbitrarily many pairs $(x,y)$
with $P\wedge R(x,y)$ with arbitrarily high probability. As in the
proof of Proposition \ref{prop:MLN-Ex-1}, 
\[
\underset{n\rightarrow\infty}{\lim}\mathcal{P}_{T_{w},n}(Q(x))=\underset{n\rightarrow\infty}{\lim}{\int}\mathrm{sigmoid}(\delta_{Q(x)}^{T_{w},n}) d{\mu_{T_{w},n}}=\underset{n\rightarrow\infty}{\lim}{\int}\mathrm{sigmoid}(w\cdot|P\wedge R(x,\mathfrak{X})|) d{\mu_{T_{w},n}} = 1
\]
\end{proof}

In DA-MLNs, the situation is reversed, and in fact, the probability
of $Q(x)$ will always tend to $\frac{1}{2}$:
\begin{proposition}
Let $T_{w}$ be the DA-MLN given by the formula $P\wedge Q(x)\wedge R(x,y):w$,
$w>0$, $|D_{x}|=|D_{y}|=n$. Then the asymptotic probability of $Q(x)$
is $\frac{1}{2}$, independent of the value of $w$.
\end{proposition}
\begin{proof}
The connection vector for the formula $P\wedge Q(x)\wedge R(x,y)$
is $(n^{2},n,1)$, and so weights will be scaled using the factor
$n^{2}$.

\begin{align*}
\underset{n\rightarrow\infty}{\lim}\mathcal{P}_{T_{w},n}(Q(x)) & =\underset{n\rightarrow\infty}{\lim}{\int}\mathrm{sigmoid}(\delta_{Q(x)}^{T_{w},n}) d{\mu_{T_{w},n}} = \underset{n\rightarrow\infty}{\lim}{\int}\mathrm{sigmoid}(\frac{w}{n^{2}}\cdot|P\wedge R(x,\mathfrak{X})|) d{\mu_{T_{w},n}}\leq\\
 & \leq\underset{n\rightarrow\infty}{\lim}{\int}\mathrm{sigmoid}(\frac{w}{n^{2}}\cdot n) d{\mu_{T_{w},n}} = \underset{n\rightarrow\infty}{\lim}\mathrm{sigmoid}(\frac{w}{n})=\mathrm{sigmoid}(0)=\frac{1}{2}
\end{align*}
\end{proof}
Again we observe that neither formalism behaves satisfactorily when scaling to larger domains. 

\subsection{Discussion \label{subsec:AggrFunc}}

When introducing DA-MLNs, Mittal et al. (2019) use the aggregation
function $\mathrm{max}$ as a pragmatic choice with good formal properties
that also proved to work well in practice. However, they also point
out that investigating different choices could be an avenue for further
research, specifically mentioning the function $\mathrm{sum}$. Therefore,
we would like to discuss briefly to what extent the issues raised
in this section depend on the precise aggregation function used. 

It seems clear that, since the asymptotic behaviour is degenerate
in those cases in which $\delta$ limits to either $0$ or $\infty$,
the order of the scaling coefficient is more important than the precise
number. When using the maximum function as an aggregation function,
the order will be the same as the highest order among the entries
of the connection vector. This scaling will be unsuitable though for
investigating the behaviour of the other literals in the formula,
since the scaling will overcompensate for the number of connections.
This is exactly what happens in the examples discussed here. Therefore,
changing the aggregation function to summation would rather exacerbate
than mitigate the issues, since the weights would then be scaled down even
further. 

Instead, it might seem plausible to use a concept of mean. Since we
are dealing with multiplicative scaling, the arithmetic mean would
not be a natural choice, and would be of the same order as the maximum.
Instead, one might try the geometric mean as an aggregation function.
This will not regulate the asymptotic behaviour of the literals at
the extremes - in fact, if one considers the standard example of $R(x)\rightarrow P:w$,
one would overcompensate for the connections of the $R(x)$ and undercompensate
for the connections of the $P$. Therefore, the geometric mean would
be far from theoretically optimal, and in fact, since the orders of
the connection numbers of the literals are different, no single aggregation
function will be adequate for all literals. However, the convergence
to the degenerate probabilities would be slowed for all literals,
and this might be relevant in practical applications.

We will now continue along a different line, and switch our representation
formalism from MLNs and undirected models, where all literals influence
each other and connections will need to be aggregated, to RLRs and
directed models, where only a single literal is being influenced and
one can scale directly along the possible connections of that single
literal.

\section{Domain-size-Aware RLR }

We have seen in the preceding sections that while DA-MLNs enhance the
dependence of limit behaviour on the weights of the formulas, there
are cases where they cannot solve the issue with regards to all queries.
More precisely, we see that the issue comes from the aggregation function
and the need to choose a single weight for formulas whose literals
have different numbers of connections. While we have also seen that
MLNs can have weight-independent limit behaviour even when there is
only at most one connection from the literal concerned, the example
given is dependent on the other literal in the formula having infinitely
many connections and thus limiting to probability 1. 

We therefore introduce a domain-size aware version of RLRs, which make use of
directed acyclic graphs to provide directionality.
After defining our version in the coming subsection, we then move on compare
it to DA-MLNs and unscaled RLRs, discuss a natural interpretation of scaled weights as proportions
and suggest a mechanism to accommodate constants in our signature. 
\subsection{Definition of Domain-size-Aware RLR}

In this subsection, we will adapt the principle of DA-MLNs to the weighted
directed model approach of RLRs, and we will see that since there is
only one literal that is directly affected by any connection (the
child literal of the edge), we can avoid the problem of aggregating
a connection vector and can scale directly by the domain sizes of
the free variables in the corresponding variable set. Formally, we
define:
\begin{definition}
\label{def:DA-RLR-} A \emph{domain-size aware Markov Logic Network (DA-RLR) }$T$\emph{ is given by the same
syntax as an RLR (see Definition \ref{def:RLRSyntax}). However, the
semantics differs as follows from the semantics given in Definition
\ref{def:RLRSemantics}:}

Let $D$ be a multi-sorted domain reduct and let $D_{x}$ be the sort
of a variable $x$. For any set of variable symbols $V$ let $|D|_{V}\coloneqq{\prod}_{x\in V}|D_{x}|$. 

Then the induction step of Definition \ref{def:RLRSemantics} is replaced
as follows:

The probability of an $\mathcal{L}_{n+1}$-structure is given by the
probability of its $\mathcal{L}_{n}$-reduct multiplied with the probability
of the groundings of the atoms in $\mathcal{L}_{n+1}\backslash\mathcal{L}_{n}$.
These are given by 
\[
\mathcal{P}(Q(\vec{x}))=\mathrm{sigmoid}(\underset{i}{\sum}\frac{w_{i}}{|D|_{V_{i}}}\underset{\vec{a}\in D}{\sum}1_{\psi_{i}(\vec{a}/V_{i})}).
\]
\end{definition}
\begin{remark}
Any variable $y$ in $V_{i}$ which does not occur in $\psi_{i}$
incrases the unscaled weight in the numerator by a factor $\left|D_{y}\right|$,
but this is exactly counteracted by increasing the scaling factor
in the denominator by the same factor. Therefore, in the context of
DA-RLRs we can always assume that $V_{i}$ is excactly all variables
in $\psi_{i}$ that do not occur in $\varphi$. 
\end{remark}
The definition now explicitly depends on the domain sizes just as
the definition of DA-MLNs depends on domain sizes. However, the number
$|D|_{V_{i}}$ in the denominator now represents the possible domains
of connection for just the atom $Q(\vec{x})$ -- since this is the
only child of that edge relation, there is no connection vector and
no aggregation function. For this representation, we can show how
the limit probability of any formula varies with the weights chosen
as domain sizes approach infinity. We suggest that a similar argument
might apply to DA-MLNs in which every entry of the connection vector
is of the same order of magnitude; however, the algorithm we give
below for determining the asymptotic behaviour of DA-RLRs depends heavily
on an underlying DAG to terminate.

\subsection{Relationship of DA-RLRs to Unscaled RLRs and to DA-MLNs}

A main feature of weight scaling approaches such as DA-RLRs or DA-MLNs is the
transparent relationship to the underlying framework. This means that,
when restricted to a single underlying domain, there is a 1-to-1 correspondence
between DA-RLRs and RLRs, defined simply by multiplying the weights
$w_{i}$ with the factor $|D|_{V_{i}}$. This has several consequences: 

Firstly, it gives a good idea of the expressivity of domain-size aware
formalisms. On a given domain, the same classes of probability distributions
over possible worlds can be expressed by DA-RLRs and by ordinary RLRs,
and one can be converted to the other transparently.
The expressivity of RLRs has been studied by Kazemi et al. (2014b) in terms of the decision thresholds that can be expressed by an RLR.\@ 
They show that all decision thresholds polynomial in the number of tuples satisfying a given formula $\phi$  can be expressed by an RLR.\@  
\begin{proposition}\label{RLR-Expressivity}
  Let $\mathcal{L}$ be the signature  $\{Q\} \cup {\{R_i\}}_{i \in I}$ with a nullary relation $Q$ and relations ${\{R_i\}}_{i \in I}$ of positive arity. Then for any polynomial $p(\vec{v})$ with terms $v_j$, each indicating a number of (tuples of) individuals for which a Boolean formula $\phi_j(\vec{x})$ of $\{R_i\}_{i \in I}$ is true or false, there is an RLR $T$ such that for any $D$, the probability of $Q$ with respect to a domain $D$ defined by $T$ is greater than 0.5 if and only if $p > 0$.   
\end{proposition}

On any given fixed domain size, then, the same holds for DA-RLRs.\@
Furthermore, Kazemi et al.\ show that any decision threshold that can be represented by an RLR is of this form:

\begin{proposition}\label{RLR-Inexpressivity}
  Let $\mathcal{L}$ be the signature  $\{Q\} \cup {\{R_i\}}_{i \in I}$ with a nullary relation $Q$ and relations ${\{R_i\}}_{i \in I}$ of positive arity, and let $T$ be an RLR over this signature. Then there is a polynomial $p(\vec{v})$ with terms $v_j$, each indicating a number of (tuples of) individuals for which a Boolean formula $\phi_j(\vec{x})$ of $\{R_i\}_{i \in I}$ is true or false, such that for any $D$, the probability of $Q$ with respect to a domain $D$ defined by $T$ is greater than 0.5 if and only if $p > 0$.   
\end{proposition}

Secondly, the 1-to-1 correspondence between RLRs and DA-RLRs on a fixed domain
means that algorithms for learning and inference developed
for the RLR framework immediately transfer to DA-RLRs.\@ This includes
the boosted structure learning approach of Ramanan et al. (2018), which was
shown there to be fully competitive with state-of-the-art structure
learning algorithms for Markov Logic Networks. It is important to
bear in mind that the semantics for (DA-)RLRs are still defined on
a given domain by grounding to a Bayesian Network; this allows us
to employ all the grounded and lifted exact and approximate inference
algorithms developed for Bayesian Networks (see the corresponding
chapters of De Raedt et al. (2016) for an overview).

On a given domain, this also means that the relationship between DA-MLNs and DA-RLRs reduces to the
relationship between MLNs and RLRs. In terms of expressivity, they are in fact incomparable.
Buchman and Poole (2015, Theorem 2) showed that MLNs cannot represent the aggregation function used in RLRs.
On the other hand, being a directed formalism, RLRs cannot deal with cyclic dependencies in an obvious way.
Therefore, the choice of representation framework (directed vs undirected) will usually depend on properties of the relationships that should be represented. However the additional consideration that (scaled directed) DA-RLRs have better asymptotic properties than (scaled undirected) DA-MLNs may sway this choice where no stronger reasons prevail.

\subsection{Interpretation of Scaled Weights and Asymptotic Behaviour}

There is a very intuitive interpretation of the $\frac{w_{i}}{|D|_{V_{i}}}{\sum}_{\vec{a}\in D}1_{\psi_{i}(\vec{a}/V_{i}))}$
from Definition \ref{def:DA-RLR-}, as this expression is clearly
equivalent to $w_{i}\frac{{\sum}_{\vec{a}\in D}1_{\psi_{i}(\vec{a}/V_{i}))}}{|D|_{V_{i}}}$.
The latter expression shows that this is just the weighted proportion
of tuples for which $\psi_{i}$ holds.

We can therefore assert the analogue of Propositions \ref{RLR-Expressivity} and \ref{RLR-Inexpressivity} for DA-RLRs, where we replace ``number of tuples'' with ``proportion'':

\begin{proposition}\label{DA-RLR-Expressivity}
  Let $\mathcal{L}$ be the signature  $\{Q\} \cup \{R_i\}_{i \in I}$ with a nullary relation $Q$ and relations $R_i$ of positive arity. Then for any polynomial $p(\vec{v})$ with terms $v_j$, each indicating a proportion of (tuples of) individuals for which a Boolean formula $\phi_j(\vec{x})$ of $\{R_i\}_{i \in I}$ is true or false, there is a DA-RLR $T$ such that for any $D$, the $T$-probability of $Q$ with respect to a domain $D$ is greater than 0.5 if and only if $p > 0$.   
Conversely, for any DA-RLR $T$ over this signature, there is such a polynomial $p(\vec{v})$, such that for any $D$, the $T$-probability of $Q$ with respect to a domain $D$ is greater than 0.5 if and only if $p > 0$.   
\end{proposition}

However, if we add proposition symbols, a DA-RLR has two
types of conditions - those depending on proportions (of formulas
with free variables) and those depending just on a Boolean true/false
value (of propositions). Asymptotically, the two types of conditions behave
very differently. To see why, consider the language consisting of
two relations $P$ and $R$, where $P$ is a 0-ary proposition and
$R$ a unary predicate. Now consider the independent distribution,
where both $P$ and $R(x)$ for any $x$ have probability $\frac{1}{2}$.
Now consider a domain size limiting to infinity. In this scenario,
the probability of $P$ will still be $\frac{1}{2}$ - in half of
the worlds it will be true, in half of the worlds it will not. The
same goes for $R(a)$ for an element $a$ that is contained in every
domain. Contrast this with the proportion of $x$ for which $R$ holds.
Asymptotically, the probability to choose a sequence of domains in
which this proportion limits to $\frac{1}{2}$ has probability $1$.
This is exactly the statement of the strong law of large numbers (see
e.\ g.\ Chapter III in the book by Bauer (1996). In other words, it is a sure
event -- the proportion will definitely be asymptotically $\frac{1}{2}$,
but one has no idea whether in a given large domain from the sequence
$P$ will hold. That is random and therefore not sure at all. 

\subsection{\label{subsec:Constants}Constants}

We close this section by briefly discussing how to incorporate constants
into the semantics of (DA-)RLRs above. The issue with adding constants
to the language is how to include them in the Relational Belief Network.
Consider for example a single-sorted signature with one binary relation
symbol $R$ and one constant $c$. A na\"{i}ve approach that simply
adds an additional node where constants fill a space in the binary
relation would then include the following nodes: $R(x,x)$, $R(c,x)$,
$R(x,c)$, $R(c,c)$. However, $R(c,c)$ is also an instance of the
other three expressions. We propose here to include constants by viewing
the interpretation of the constants as additional and outside of the
domain and thus replace $R(c,x)$ and $R(x,c)$ by new unary relation
symbols $R_{c,\_}(x)$ and $R_{\_,c}(x)$ and $R(c,c)$ with a proposition
$R_{c,c}$. This allows us to fix the probabilities and the effects
of constants in relations separately (the motivation for including
constants in the first place). Not explicitly including constants
in our language has the additional benefit that all domain elements
(which are now unnamed) are treated symmetrically by the RLR formalisms,
so that we can evaluate the probability $\mathcal{P}(R(x))$ without
specifiying which domain element instantiates $x$. However, if we
assume nothing to be known about the constants, we can adapt a given
DA-RLR to reflect this fact. 
\begin{definition}
The \emph{generic extension }of a given DA-RLR by constants $\vec{a}$
is derived from the DA-RLR as follows:

Add the nodes for the new relation symbols introduced by substituting
the constant at appropriate (sort-matching) places. 

Let $\left(R(\vec{x},\vec{y}),\left(w_{i},\psi_{i}(\vec{x},\vec{y},\vec{z})\right)_{i}\right)$
be the label of $R$, where $\vec{x}$ are in the places into which
$\vec{a}$ is substituted in $R_{\vec{a}}$. The label of $R_{\vec{a}}$
is given by $\left(R_{\vec{a}}(\vec{y}),\left(w_{i},\psi_{i}'(\vec{y},\vec{z})\right)_{i}\right)$,
where $\psi_{i}'$ is obtained from $\psi_{i}$ by replacing all occurrences
of $Q(\vec{x},\vec{y},\vec{z})$ for a relation symbol $Q$ with $Q_{\vec{a}}(\vec{y},\vec{z})$. 

The parents of $R_{\vec{a}}$ are exactly those relation symbols that
occurr in a $\psi_{i}'$. 
\end{definition}
When discussing asymptotic probabilities below, the proportion of
a domain that is named by constants in the signature approaches $0$
as the domain size increases, so whether they are counted as part
of the domain or not is immaterial for the limit behaviour. 

To formulate this, we briefly introduce some notation: 
\begin{definition}
A sequence of domain choices $(D_{n})_{n\in\mathbb{N}}$ is called
\emph{ascending }if for all $n\in\mathbb{N}$, $\left(D_{n}\right)_{S}\subseteq\left(D_{n+1}\right)_{S}$
for every sort of the signature. A sequence of $\mathcal{L}$-structures
$\mathfrak{X}_{n}$ on $D_{n}$ is called \emph{ascending }if the
interpretations of every relation symbol in $\mathcal{L}$ agree on
$\mathfrak{X}_{n}$ and $\mathfrak{X}_{n+1}$ whenever all entries
of a tuple are taken from $D_{n}$. 
\end{definition}
\begin{lemma}
Let $T$ be a DA-RLR over the signature $\mathcal{L}$ and let $(D_{n})_{n\in\mathbb{N}}$
be an ascending sequence of choices of domains for the sorts of $\mathcal{L}$,
such that the domain sizes of every sort are unbounded. Then for every
relation symbol $R(\vec{x})$, the limit of the probability of $R(\vec{x})$
with respect to $D_{n}$ as $n\rightarrow\infty$, in $T$ is the
same as the corresponding limit probability of $R_{\vec{a}}$ in $T_{\vec{a}}$,
the generic extension of $T$ by appropriate constants $\vec{a}$. 
\end{lemma}
\begin{proof}
  The probabilities encoded in the DA-RLR $T$ are given as
  \[
  \mathrm{sigmoid}\left(\underset{i}{\sum}\frac{w_{i}}{|D|_{V_{i}}}\underset{\vec{a}\in D}{\sum}1_{\psi_{i}(\vec{a}/V_{i}))}\right)=\mathrm{sigmoid}\left(\underset{i}{\sum}w_{i}\underset{\vec{a}\in D}{\sum}\frac{1_{\psi_{i}(\vec{a}/V_{i}))}}{|D|_{V_{i}}}\right).
  \]
In the DA-RLR $T_{\vec{a}}$, it is easy to see that the calculation
is the same, with the sole exception that the constants are not counted
in the calculation of the proportion. However, the contribution of
the constants limit to $0$ with increasing domain size, and since
the $\mathrm{sigmoid}$ function is continuous, this implies the statement
of the lemma. 
\end{proof}

We close this section with two examples of networks obtained by a
generic extension:
\begin{example}
Consider the two RLRs of earlier examples, both with the RBN $R\rightarrow Q$
and the labels $\varphi=Q(x),\psi=R(x)$ and $\varphi=Q(x),\psi=R(y)$
respectively. Assume further that we are extending them generically
by a constant $a$. Then those extensions have different RBNs: The
former has the RBN 
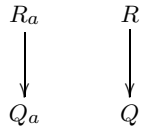
\[
\xymatrix{R_{a}\ar[d] & R\ar[d]\\
Q_{a} & Q
}
\]
while the latter has 
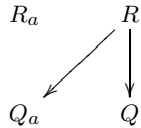
\[
\xymatrix{R_{a} & R\ar[d]\ar[dl]\\
Q_{a} & Q
}
\]
because substituting $a$ for $x$ in $R(x)$ results in $R(a)=R_{a}$,
while substituting $a$ for $x$ in $R(y)$ leaves $R(y)$ unchanged.
\end{example}

\section{Asymptotic Probabilities}

In this section we provide an algorithm for calculating the asymptotic
probabilities of a DA-RLR as domain sizes increase.

Schulte et al. (2014) introduce a semantics for Bayesian Networks
that are specifically designed abstractly to reason about proportions
(or \emph{class-level probabilities} as they are referred to there)
rather than with respect to a specific grounding. However, since the
probabilities and their depedencies are interpreted as random substitutions
rather than taking class-wise probabilities into account directly,
this limits them to modelling a linear relationship, rather than the
sigmoidal relationship that is encoded by DA-RLRs.\ Therefore, we will
have to choose a somewhat different formalism, which we will introduce
in several stages. 

It is based on the observation (formally stated by Halpern (1990)
and also utilised by Schulte et al. (2014)) that the proportion of
domain elements $x$ satisfying $\psi(x)$ can be interpreted as the
probability of $a$ satisfying $\psi(a)$, where $a$ is a new constant
that we have no additional information about. While this is ordinarily
only true on a given structure, since proportion is only defined with
respect to a given structure, we can then resort to the law of large
numbers to show that the variation between structures satisfying the
same propositions vanishes as domain size increases.
We will therefore recur to the method of encoding constants using propositions
introduced in Subsection \ref{subsec:Constants} above. 

We first provide an outline of the algorithm in words, and then corresponding
pseudocode.

We will simultaneously compute two separate but interlinked quantities:
\begin{enumerate}
\item A probability distribution over the propositions ocurring in a given
DA-RLR, which models the asymptotic probability distribution of those
propositions as domain sizes increase;
more precisely, given a set of values $C$ for (some of the) propositions ocurring in the DA-RLR,
a number between $[0,1]$. 
\item For any formula $\psi$ and any complete set of values $C$ for all
the propositions $P$ with $\mathrm{index}(P)\leq\mathrm{index}(\psi)$
in the DA-RLR, a number $p_{\psi}\in[0,1]$ which models the asymptotic
proportion of tuples $\vec{a}$ satisfying $\psi(\vec{a})$ from any
domain for which $C$ is true as domain sizes increase. More precisely,
with probability 1 a sequence of domains whose domain sizes approach
infinity and which satisfy $C$ has proportions tending to $p_{\psi}$.
\end{enumerate}

The algorithm proceeds as follows:

$p_{\psi,C}$: For any $\psi(\vec{x})$ and any $C$,
we compute $p_{\psi(\vec{x})}$ by adding new constants $\vec{a}$
to the language, whose number and sorts match $\vec{x}$, and set
$p_{\psi(\vec{x})}$ to be the probability of $\psi(\vec{a})$ conditioned
on $C$, computed with respect to the generic extension of the DA-RLR
by $\vec{a}$. As $\psi(\vec{a})$ is a Boolean combination of atoms
without free variables, which are propositions in the generic extension,
this reduces to knowing the probability distribution of those propositions.
Note that the index of all propositions required for this computation
is not more than the index of $\psi$.

Propositions: Computing the probability distributions of the propositions
can be done in a sequential manner by ordering the propositions $P_{i}$
by index and then computing the probability of $P_{i}=X_{i}$ conditioned
on $P_{j}=X_{j}$ for all $j<i$. So assume a choice $C$ of values
of all propositions of lower index than $P_{i}$ are given. Then,
by the algorithm for $p_{\psi,C}$ we can also compute $p_{\psi,C}$
for all $\psi$ of lower index than $P_{i}$. Now we can define the
probability of $P_{i}$ given $C$ to be $\mathrm{sigmoid}({\sum}_{k}w_{k}\cdot p_{\psi_{k},C})$
where $k$ varies over the labels of $P_{i}$. 

Since all the $\psi_{k}$ have lower index than $P_{i}$, the algorithm
terminates when it reaches index $0$ (root nodes). 
\begin{algorithm}\label{algo:}
  \caption{Functional Pseudocode}
\begin{align*}
&\texttt{cond\_probability }  \texttt{:: (Proposition, Values of propositions of lower order, DA-RLR) -> Double}\\
&\texttt{cond\_probability (p, c, t) }  \texttt{= sigmoid( sum([w * proportion (psi,c,t)| (w,psi) in labels (p,t)]) }\\
\\
&\texttt{proportion }  \texttt{:: (Formula, Values of propositions of equal or lower order, DA-RLR) -> Double}\\
&\texttt{proportion(psi, c, t) }  \texttt{= probability(substitution (psi,as), c, generic (t, as))}\\
& \texttt{where as = free\_variables (psi)}\\
\\
&\texttt{labels } \texttt{:: (Relation Symbol, DA-RLR) -> [(Double, Formula)]}\\
&\texttt{/* retrieves the } \texttt{labels from the DA-RLR */}\\
\\
&\texttt{probability } \texttt{:: (Formula of propositions, (Partial) values of propositions, DA-RLR) -> Double}\\
&\texttt{/* computes the }  \texttt{conditional probabilities using iterated calls to cond\_probability */}\\
\\
&\texttt{substitution } \texttt{:: (Formula, [Sort]) -> Formula of Propositions}\\
&\texttt{/* substitutes a } \texttt{sort-matching list of constants for the free variables of the formula */}\\
\\
&\texttt{free\_variables }  \texttt{:: Formula -> [Sort]}\\
&\texttt{/* returns the }  \texttt{sorts of the free variables in the formula */}\\
\\
&\texttt{generic }  \texttt{:: (DA-RLR, [Sort]) -> DA-RLR}\\
&\texttt{/* returns the } \texttt{generic extension as discussed in Subsection \ref{subsec:Constants} */}
\end{align*}
\end{algorithm}
Using the law of large numbers, we perform induction on the index
to show that the algorithm is indeed sound:
\begin{theorem}
Let $T$ be a DA-RLR over the signature $\mathcal{L}$ and let ${(D_{n})}_{n\in\mathbb{N}}$
be an ascending sequence of choices of domains for the sorts of $\mathcal{L}$,
such that the domain sizes of every sort are unbounded. Then the following
hold:

The limit of the probabilities of any formula containing just the
propositions in $\mathcal{L}$ as $n\rightarrow\infty$ is given by
the probability distribution calculated in Algorithm \ref{algo:}. 

For any quantifier-free $\mathcal{L}$-formula $\psi$, the proportion
of tuples of a randomly chosen ascending sequence of $\mathcal{L}$-structures
$\mathfrak{X}_{n}$ on $D_{n}$ satisfying a given set of values $C$
for the propositions in $\mathcal{L}$, limits to $p_{\psi,C}$ almost
\end{theorem}
\begin{proof}
We prove the theorem by induction on the index, simultaneously for
both the probabilities of the propositions and the proportions of the formulas.

As our goal is to use the law of large numbers, we restrict the formulas we consider to
complete descriptions $\psi$ of $\vec{x}$, that is,
formulas which determine for each relation symbol $R$ and sort-appropriate tuple $\vec{z}$ from $\vec{x}$ whether $R(\vec{z})$ holds or not.
This includes the propositions, which take the empty tuple.
Every quantifier-formula can be expressed as a disjunction of mutually incompatible
complete descriptions, and hence their proportions can simply be computed as the sum of the
proportions of their associated complete descriptions.

Furthermore, to avoid issues arising from tuples with overlapping entries, we show a slightly stronger form of the theorem which allows us to proceed variable by variable. 

Let  $\psi(\vec{x},\vec{y})$ be a complete desciption in the variables $\vec{x},\vec{y}$,
let $\varphi$ be the complete description of $\vec{x}$ implied by $\psi$ and let $C$ be the choice of values for the propositions that is implied by $\psi$.
Let \[
p_{\psi,\varphi} := \frac{p_{\psi,C}}{p_{\varphi,C}}.
\] 
Then for any $\delta,\varepsilon,\zeta > 0$
and any ascending sequence of domain choices  $(D_n)$,
there is an $N \in \mathbb{N}$ such that for all $n > N$, 
a randomly chosen sequence of $\mathcal{L}$-structures
$\mathfrak{X}_{n}$ on $D_{n}$ satisfying a given set of truth values $\varphi(\vec{x})$
for the quantifier-free formulas in $\mathcal{L}$ with free variables among $\vec{x}$,
with probability at least $1 - \zeta$,
for a proportion of at least $1 - \varepsilon$ of sort-appropriate tuples for $\vec{x}$ in $D_n$ that satisfy $\varphi$,
the proportion of sort-appropriate tuples for  $\vec{y}$ in $D_n$ that satisfy $\psi(\vec{x},\vec{y})$
lies between $p_{\psi,\varphi} - \delta$ and $p_{\psi,\varphi} + \delta$. 

Iteratively using Bayes'rule for the proportions we are computing,
we can now restrict our attention to complete descriptions of the form  $\psi(\vec{x},y)$,
where $y$ is a single variable. 

Root nodes: The probability of any proposition $P$ at the root is
given by $\mathrm{sigmoid}(w)$, which is exactly the value used in
the algorithm too. Since all root propositions are independent of
each other, the probability of Boolean combinations is computed in
a straightforward way using the axioms of probability. 

For any formula $\psi(\vec{x})$ using only relation symbols at the
root nodes, the probability of any given tuple $\vec{a}$ to satisfy
$\psi(\vec{a})$ is as given by the algorithm (as we have just seen).
Furthermore, for any ${\vec{a},b_1}$ and ${\vec{a},b_2}$ in $D$,
$\psi(\vec{a},b_1)$ and $\psi(\vec{a},b_2)$ are independent of each
other when conditioned on the implied complete description $\psi'(\vec{x})$ of $\vec{a}$.
Therefore, the requirements of the law of large numbers are
satisfied, and the proportion of tuples $\vec{x}$ for which $\psi(\vec{x})$
holds limits to $p_{\psi,\psi'}$ almost surely.

Index $i+1$: The probability of any proposition $P$ given a reduct
to the language of index $i$ is given by $\mathrm{sigmoid}({\sum}_{k}w_{k}\cdot(\textrm{proportion of tuples satisfying }\psi_{k}))$,
where the index of each $\psi_{k}$ is $i$ or less. By the induction
hypothesis, the proportion of tuples satisfying $\psi_{k}$ limits
to $p_{\psi,C}$ almost surely. Since the sigmoid function is continuous,
this implies that the probability of $P$ limits to
$\mathrm{sigmoid}({\sum}_{k}w_{k}\cdot p_{\psi_{k},C})$ as computed by the algorithm.

For any formula $\psi(\vec{x})$ of index $i+1$, the asymptotic probability
of any tuple $\vec{a}$ to satisfy $\psi(\vec{a})$ is given by the
algorithm (as we have just seen).

As for the initial case, we have to argue that the events $\psi(\vec{a},b_1)$
and $\psi(\vec{a},b_2)$ are independent for $\vec{a},b_1, b_2 \in D$.
We can verify this, however, by adding both $\vec{b_1}$
and $\vec{b_2}$ as constants to the RLR and considering the generic
extension. Then we can see that all joint ancestors of both $\psi(\vec{a},b_1)$
and $\psi(\vec{a},b_2)$ are from the original RLR.\ In particular, they
are all either included in $C$ or their proportions can be calculated from $\psi'$
by the induction hypothesis,
to any desired level of accuracy on an arbitrarily large proportion of tuples $\vec{a}$.
We can now restrict our attention to those tuples and see $\psi(\vec{a},b_1)$ and $\psi(\vec{a},b_2)$ as independent after conditioning on $\psi'$, since they can be approximated from above and below to arbitrary precision by independent sequences ofrandom variables.

This allows us to invoke the law of large numbers again and
to conclude that the proportion of tuples $\vec{a}$ satisfying $\psi$
limits to $p_{\psi,\psi'}$ almost surely. 
\end{proof}

\section{\label{sec:Discussion}Discussion and Mixed RLR}

Jain et al. (2010), who were the first to discuss adapting the weights
to the population size, advocate learning the size-to-weight function
from datasets of different sizes. In this section, we will use the
interpretation we have given of scaled weights as proportions (rather
than absolute numbers of incidences) to argue that the semantics of
the use case can give an indication as to whether scaling is appropriate
or not. 

\subsection{\label{subsec:Interpretation-of-scaling}Interpretation of Weight
Scaling }

The key is that we have to assess whether the dependency we stipulate
by assigning a weight to a formula shows a dependence on absolute
numbers of an associated event or a dependence on the proportion of
possible events that satisfy the criteria. 

To see the difference, imagine the following scenario: We would like
to model how much a student has learned depending on whether the teaching
he has received has been good or poor. Say that we decide to use a
vocabulary consisting of two domains whose individuals are ``lessons''
and ``students'' and then two predicates, a unary predicate ``good
lesson'' $G(x$) ranging over the first domain and a unary predicate
``learning success'' $L(y)$ ranging over the second. We could then
frame this question in different ways. First, assume that we would
like to evaluate the \emph{effectiveness} of the teaching - has a
student learnt sufficiently much \emph{considering the amount of time
he spent being taught}. In this case, it seems reasonable to assume
that this will depend on the \emph{proportion }of lessons that have
been good. Whether a student is in education for 12 years or 9 years,
we would still believe that  the effectiveness of teaching depends
on its quality. However, we could also evaluate the sheer amount of
learning a student has received - then $L(x)$ might represent a certain
fixed skill level, like ``can read and write''. Now it would not
seem sensible to use a proportionalist model: instead, it seems rather
reasonable to believe that in the limit of more and more lessons,
the student will eventually have attended enough good ones to have
learnt the skill. 

Bearing this distinction in mind, we will evaluate how our findings
relate to the different scenarios  discussed in
(Kazemi et al., 2014b).

\subsection{\label{subsec:KazEx}A Closer Look at the Examples from Kazemi et
al. (2014b)}

In their introduction, Kazemi et al. (2014b) list a number of different
ways in which changing population sizes might be relevant in an AI
model. In the first scenario, elaborated in (Poole, 2003), the likelihood
of someone having committed a crime is dependent on how many other
people fit the description of the criminal. One way to implement this
would be as follows: There is one domain of suspects, a unary predicate $C(x)$
for being a criminal and a unary predicate $D(x)$ for matching the
description of the criminal. We then have propositions $C_j$ and $D_j$ for Joe being a criminal and matching the description respectively.
This is in the spirit of introducing a constant $j$ for Joe and proceeding in the manner of Subsection \ref{subsec:Constants}.
The RBN could be as given below,
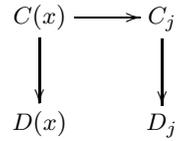
\[
\xymatrix{C(x)\ar[r]\ar[d] & C_j\ar[d]\\
D(x) & D_j
}
\]
and the formulas involved would be atomic with a positive weight attached
to $C(x)$ and $C_j$ when evaluating $D(x)$ and $D_j$ repectively
(as well as a base weighting to reflect the inherent likelihood of
$D(x)$ independent of the crime) and a negative weight given to $C(x)$
when evaluating $C_j$. 

Note that we need the separate treatment for Joe since a general model, 
everyone being less likely to be the criminal if another person
is the criminal would introduce a cycle into the RBN. 

In this model, scaling of weights could only be considered at the
edge $C(x)$ to $C_j$, as this is the only one with parent variables
that are not in the child. However, here scaling is counterproductive:
The edge is intended to model that there is likely to be just one
criminal, and that if there is one there is unlikely to be another.
So here the model as it stands seems well equipped to deal with the
issue of varying domain size without any scaling. If there are more
people that are a priori equally likely to have committed the crime,
then the likelihood for Joe having committed it is smaller. Kazemi
et al. (2014b) introduce a different kind of varying domain size:
they say that it is arbitrary which population we base our model on,
whether it is the neighbourhood, the city or the whole country. However,
we think that which scale to use should not be arbitrary, since any
model will rely on the assumption that everyone is equally likely
to have committed the crime. In a gang stabbing, the population should
be restricted to the gang membership; in an internet-based case of
credit card fraud, one might have to consider the whole world. Another
option of dealing with this would be to have an arbitrary population
but then adapt the description predicate $D(x)$ to include ``lives
in the area of the crime''. In that case, this weight should indeed
be scaled by population size, at least if we assume that the population
we choose definitely encompasses at least everyone living in the area
of the crime. This is, in a sense, the opposite scaling of what has
been suggested here: We are scaling precisely to make the likelihood
of $D(x)$ limit to 0 independent of the chosen weights since we are
intending to condition on $D_j$ anyway when evaluating the RBN later.
While this is a very interesting phenomenon, exploring its technical
background is outside the scope of the present work.

In the second example, Kazemi et al. (2014b) mention a situation in
which the population is variable, such as the population of a neighbourhood
or of a school class. This is much like the example we have discussed
above in relation to learning success: While the number of lessons
a student takes varies between students, how we want to deal with
this situation depends on exactly what sort of question we are asking. 

\subsection{Mixed RLR}

The examples above show that the decision to use or not to use proportional
scaling in a model depends on exactly the sort of questions we are
asking of the model and also why exactly the population is varying
in the first place. It might also very well be appropriate to use
proportional scaling on some of the connections but not on others, leading to \emph{mixed RLR}.
Consider for instance an application that models pollution in a lake.
The model has two domains, one for tributaries to the lake and one for human
users of the lake. The signature has two unary predicates, $R(x)$
signifying that the water from tributary $x$ is polluted, and $H(y)$,
meaning that human $y$ pollutes the lake, as well as a proposition
$P$ meaning that the lake is polluted. If we were to assume that
pollution in the lake depends on the \emph{proportion} of incoming
water that is polluted and the \emph{amount} of pollution added to
that by humans, we could mark the formula involving $R(x)$ as proportional
while keeping the formula involving $H(y)$ as absolute. Mixed
RLRs could thus be useful to model situations in which connections
have different quality, and could simply be represented by marking
some formulas to be proportional and some not. As the overall probabilities
are obtained from the individual weights using logistic regression,
one can simply adapt some of the individual weights as the domain
size changes, while leaving others unchanged.

Concerning the expressiveness of mixed RLRs, we can again turn to Propositions \ref{RLR-Expressivity} and \ref{RLR-Inexpressivity}.
However, since we adapt weights that appear as the coefficients of the polynomials $p$ mentioned there, we cannot express mixed terms containing both proportions and  raw numbers in decision thresholds.

\begin{proposition}\label{Mixed-RLR-Expressivity}
  Let $\mathcal{L}$ be the signature  $\{Q\} \cup \{R_i\}_{i \in I}$ with a nullary relation $Q$ and relations $R_i$ of positive arity. Then for any pair of polynomials $p(\vec{v})$ and $q(\vec{r})$ with terms $v_j$, each indicating a proportion of (tuples of) individuals for which a Boolean formula $\phi_j(\vec{x})$ of $\{R_i\}_{i \in I}$ is true or false, and $r_k$, each indicating a number of (tuples of) individuals for which a Boolean formula $\psi_j(\vec{x})$ of $\{R_i\}_{i \in I}$ is true or false, there is a mixed RLR $T$ such that for any $D$, the $T$-probability of $Q$ with respect to a domain $D$ is greater than 0.5 if and only if $p + q > 0$.   
Conversely, for any DA-RLR $T$ over this signature, there are such polynomials $p(\vec{v})$ and  $q(\vec{r})$, such that for any $D$, the $T$-probability of $Q$ with respect to a domain $D$ is greater than 0.5 if and only if $p + q> 0$.   
\end{proposition}

\subsection{\label{subsec:Random-Sampling} Random Sampling}




While random sampling from subpopulations has already been considered as a means for scalable learning in  Subsection \ref{subsec:Extrapolation} above, it can also occur in naturally variable domains.

By passing from RLRs to DA-RLRs, we move from considering absolute numbers of parent
atoms (which are heavily distorted by random sampling) to proportions,
which should be approximately conserved. For a concrete example, let
us consider the model of teaching and learning from Subsection \ref{subsec:Interpretation-of-scaling}.
Assume we had decided to consider absolute learning success, which
usually would not suggest scaling by domain size. However, we can
only observe a limited sample of lessons, and how many that is varies
from school to school. Now if we were to estimate learning success
here, it is natural to make it dependent on the \emph{proportion }and
not the \emph{number }of good lessons that we are seeing.

Random sampling is prevalent throughout the natural and social sciences,
however, and for instance the drug study example of Kazemi et al.
(2014b) also falls under this category.

\subsection{Projective Families of Distributions}

Jaeger and Schulte (2018, 2020) discussed the concept of projectivity and supplied a complete characterisation of projective families of distributions in terms of exchangeable arrays.
To formulate it, we consider an \emph{embedding} $\iota:D \rightarrow D'$ of (potentially multi-sorted) domains to be a set of injective maps between the matching sorts of $D$ an $D'$.
For every such embedding $\iota:D \rightarrow D'$ and every $\mathcal{L}$-structure $\mathfrak{X}'$ on the domain $D'$, we define $\mathfrak{X}'_{\iota}$ to be the $\mathcal{L}$-structure on the domain $D$ for which a relation $R$ holds for $\vec{x} \in D$ if and only if $R$ holds for $\vec{\iota(x)}$ in $\mathfrak{X}$. 

\begin{definition}
  Let $(\mathcal{P}_{T,D})_D$ be a family of probability distributions over the possible $\mathcal{L}$-structures on the domain $D$.
  $(\mathcal{P}_{T,D})_D$ is called \emph{projective} if for all injective maps $\iota:D \rightarrow D'$ and all $\mathcal{L}$-structures with domain $D'$ the following holds:
  \begin{align*}
    \mathcal{P}_{T,D}(\mathfrak{X}) = \mathcal{P}_{T,D'}(\{\mathfrak{X}' | \mathfrak{X}'_{\iota} = \mathfrak{X}\})
  \end{align*}
\end{definition}

The importance of projectivity stems from the fact that for a $\iota:D \rightarrow D'$ and any quantifier-free $\mathcal{L}$-formula $\varphi$, the $(\mathcal{P}_{T,D})$-probability that $\varphi(\vec{a})$ holds for a given $\vec{a} \in D$ is the same as the  $(\mathcal{P}_{T,D'})$-probability that $\varphi(\vec{\iota(a)})$ holds.
In particular, this means that the probabilities of a quantifier-free ground query that only mentions $m << n$ elements of a domain of size $n$ can actually be evaluated in the domain of size $m$ consisting only of those $m$ elements. Therefore, the complexity of such  queries is constant with respect to domain size.
As discussed by Jaeger and Schulte in (2018), under certain mild additional assumptions, projectivity also guarantees statistical consistency of learning parameter learning on randomly sampled subsets. This property relates projectivity directly to the task addressed by the present article.

While projective families of distributions have very attractive properties, their expressivity is limited. In particular, the fragments of common frameworks identified as projective by Jaeger and Schulte (2018) do not involve any interaction with a potentially unbounded number of other domain elements.
In the case of RLRs, this means in the notation of Definition \ref{def:RLRSyntax} that every variable occurring in a formula $\psi_i$ occurs in the corresponding $\varphi$. In fact, they are exactly the RLRs for which the variable set $V_i$ is empty and therefore the DA-RLR semantics coincides with the unscaled semantics of RLRs.
In this paper, we analyse the asymptotic behaviour of DA-RLRs and conclude from there that the asymptotic probabilities are well-defined and depend meaningfully on the input weights. However, unlike in  aprojective family of distributions, we do not claim that the asymptotic probabilities are actually equal to the probabilities on a given fixed domain.
Indeed, it is easy to see that general DA-RLRs are not projective.
\begin{proposition}
  There is a non-projective DA-RLR. 
\end{proposition}
\begin{proof}
  Consider the DA-RLR $T$ of Example \ref{ExampleRLR}, with a signature $\mathcal{L} := \{Q,R\}$, an RBN $R \longrightarrow Q$, weight $w_R = 0$ for $R$, $\varphi = Q(x)$, $\psi = R(y)$ and $w_Q = 1$.

  We consider a domain $D$ with a single element $a$, and the quantifier-free query formula $\chi := Q(a) \wedge R(a)$. Then $\mathcal{P}_{T,D}(\chi) = \mathcal{P}_{T,D}(R(a)) \cdot \mathcal{P}_{T,D}(Q(a)|R(a))$  by the Bayesian formula. By the definition of the DA-RLR semantics, we obtain  $\mathcal{P}_{T,D}(R(a)) = 0.5$ and $\mathcal{P}_{T,D}(Q(a)|R(a)) = \mathrm{sigmoid}(1) \approx 0.73$, and therefore  $\mathcal{P}_{T,D}(\chi) \approx 0.37$.

  Now consider a domain $D'$ with two elements $a$ and $b$. Then we can assert as before that $\mathcal{P}_{T,D'}(\chi) = \mathcal{P}_{T,D'}(R(a)) \cdot \mathcal{P}_{T,D'}(Q(a)|R(a))$ and that $\mathcal{P}_{T,D'}(R(a)) = 0.5$. However, there are now 2 possibilities for $R(b)$: With a 0.5 probability, $R(b)$ holds. For that case, $\mathcal{P}_{T,D'}(Q(a)|R(a),R(b)) = \mathrm{sigmoid}(1) \approx 0.73$.
  With a 0.5 probability, though, $R(b)$ does not hold. Then $\mathcal{P}_{T,D'}(Q(a)|R(a),\neg R(b)) = \mathrm{sigmoid}(0.5) \approx 0.62$ Overall, we can conclude that
  $\mathcal{P}_{T,D}(\chi) \approx 0.5 \cdot (0.5 \cdot 0.73 + 0.5 \cdot 0.62) < 0.37$. Therefore, $(\mathcal{P}_{T,D})_D$ is not projective.
\end{proof}

In other words, we unlock the much greater expressivity of DA-RLRs, which can model interactions involving potentially unbounded numbers of elements, by giving up the requirement that probabilities are constant on all domains in favour of the requirement that they are asymptotically convergent in a meaningful way.

\section{\label{sec:FutWork}Conclusion}

We have seen that DA-RLRs provide a framework for statistical relational AI that inherits the expressiveness of -- and the inference and learning algorithms available for -- RLRs but takes into account domain size when interpreting the weight parameters.
By providing an algorithm to compute the asymptotic probabilities in DA-RLRs, we unlock their use for
asypmptotically sound learning and inference.
We also show that the asymptotic probabilities obtained
in this way depend meaningfully on the weight parameters supplied.
This improves on the properties of DA-MLNs, which are shown to lack asymptotic dependence on the weight parameters in several paradigmatic cases.
Our results are optimal in the sense that DA-RLRs do not possess full projectivity.
Since the scaling of weights is such a transparent operation, we can also introduce mixed RLRs, which have both scaled and unscaled aggregators. This promises to enhance modelling capabilities for a large class of practical examples, whose behaviour across domain sizes does not conform to a polynomial increase in decision threshold. 

Beyond DA-RLRs themselves, our work could have wider implications on other directed formalisms such as probabilistic logic programming.
While probabilistic logic programming does not employ weights and it is therefore less obvious how to apply the scaling factor, our work demonstrates that directed approaches as such have a distinct advantage when tackling domain size extrapolation.
Since DA-RLRs are ultimately rooted in logistic regression, they could also utilise the advances and extensions of that field, such as multinomial regression for multivalued data. Any relational analogue of such a derived formalism will probably be able to utilise the results of this paper with little additional effort.

The methods developed in this paper for the analysis of MLNs and DA-MLNs also promises further applications to the study of undirected formalisms, which are still highly popular tools for statistical relational learning and knowledge representation. Such considerations could help researchers move beyond the heavy computational machinery in current use and towards a more transparent and rigorous probability-theoretic treatment.

\begin{acknowledgements}
The author would like to thank his students Michael Vogt and Omar AbdelWanis for comments
and suggestions on an earlier version of the manuscript. Thanks are also due to the anonymous
reviewers, which significantly improved the article. 
\end{acknowledgements}

\section*{Declarations}
Work on this article was not supported by any dedicated funding outside the regular means of the university. The author does not have any conflict of interest regarding this work.
Ethics approval, consent to participate, consent for publication and availability of data, material and code are not applicable to this work. This is the sole work of the single author, Felix Weitkämper.

\end{document}